\documentclass[letterpaper,10pt,conference]{ieeeconf}

\IEEEoverridecommandlockouts 
\usepackage[T1]{fontenc}

\usepackage{amsmath,amssymb}
\usepackage{algpseudocode}
\usepackage{algorithm}
\usepackage{marginnote}
\usepackage{caption}
\usepackage{todonotes}
\usepackage{tikz}
\usetikzlibrary{arrows,positioning}
\usepackage{tikz-cd}
\usepackage[justification=justified,font=footnotesize]{caption}
\usepackage{subcaption}
\usepackage{graphicx}
\usepackage{array}
\newcolumntype{M}[1]{>{\centering\arraybackslash}m{#1}}
\usepackage{comment}
\usepackage{color}  
\usepackage[noadjust]{cite}

\newtheorem{thm}{Theorem}
\newtheorem{cor}{Corollary}

\usepackage{hyperref}

\algnewcommand\algorithmicforeach{\textbf{for each}}
\algdef{S}[FOR]{ForEach}[1]{\algorithmicforeach\ #1\ \algorithmicdo}

\begin{document}

\title{Technical Report for Real-Time Certified Probabilistic Pedestrian Forecasting}

\author{Henry O. Jacobs, Owen K. Hughes, Matthew Johnson-Roberson, and Ram Vasudevan
 \thanks{ H.O. Jacobs, O. Hughes, M. Johnson-Roberson, and R.~Vasudevan are with the University of Michigan, Ann Arbor, MI 48109
{\scriptsize \{\texttt{hojacobs,owhughes,mattjr,ramv}\}~\texttt{@umich.edu}}}
\thanks{This material is based upon work supported by Ford Motor Company and the National Science Foundation under Grant No. 1562612.}
}


\maketitle

\begin{abstract}
The success of autonomous systems will depend upon their ability to safely navigate human-centric environments.
This motivates the need for a real-time, probabilistic forecasting algorithm for pedestrians, cyclists, and other agents since these predictions will form a necessary step in assessing the risk of any action.
This paper presents a novel approach to probabilistic forecasting for pedestrians based on weighted sums of ordinary differential equations that are learned from historical trajectory information within a fixed scene.
The resulting algorithm is embarrassingly parallel and is able to work at real-time speeds using a naive Python implementation.
The quality of predicted locations of agents generated by the proposed algorithm is validated on a variety of examples and considerably higher than existing state of the art approaches over long time horizons.
\end{abstract}


\section{Introduction}

Autonomous systems are increasingly being deployed in and around humans. 
The ability to accurately model and anticipate human behavior is critical to maximizing safety, confidence, and effectiveness of these systems in human-centric environments. 
The stochasticity of humans necessitates a probabilistic approach to capture the likelihood of an action over a set of possible behaviors. 
Since the set of plausible human behaviors is vast, this work focuses on anticipating the possible future locations of pedestrians within a bounded area. 
This problem is critical in many application domains such as enabling personal robots to navigate in crowded environments, managing pedestrian flow in smart cities, and synthesizing safe controllers for autonomous vehicles (AV).  

With a particular focus on the AV application several additional design criteria become important. 
First, false negative rates for unoccupied regions must be minimized.
The misclassification of space in this way has obvious safety issues.
Second, speed is paramount.
To effectively use human prediction within a vehicle control loop prediction rates must be commensurate with the speed at which humans change trajectories.
Finally, long-time horizon forecasting is preferable since this improves robot predictability, reduces operation close to safety margins, prevents the need for overly conservative or aggressive controllers and makes high-level goal planning more feasible.
This paper presents an algorithm for real-time, long-term prediction of pedestrian behavior which can subsequently be used by autonomous agents.
As depicted in Figure \ref{fig:gates-1-2}, this method works quickly to generate predictions that are precise while reducing the likelihood of false negative detections.

\begin{figure}
	\centering
	\begin{subfigure}[b]{.45\linewidth}
		\includegraphics[width=\linewidth]{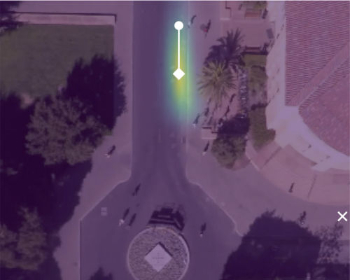}
		\caption{$t=1.83s$}
	\end{subfigure}
		\begin{subfigure}[b]{.45\linewidth}
			\includegraphics[width=\linewidth]{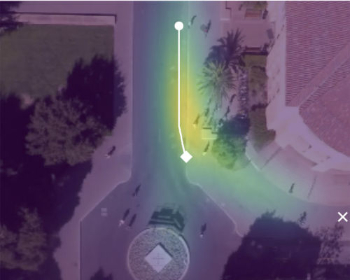}
			\caption{$t=4.83$}
		\end{subfigure}
	
	\begin{subfigure}[b]{.45\linewidth}
		\includegraphics[width=\linewidth]{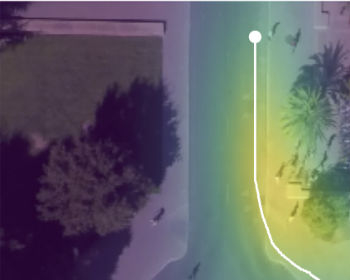}
		\caption{$t=7.83s$}
	\end{subfigure}
	\begin{subfigure}[b]{.45\linewidth}
		\includegraphics[width=\linewidth]{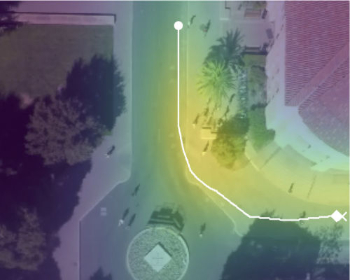}
		\caption{$t=11.5s$}
	\end{subfigure}
	\caption{The performance of the presented algorithm captures the most probable routes that a pedestrian chooses. The dot is the starting point of the trajectory, the diamond is the position at time $t$, and the X is the end of the trajectory. The likelihood of detection is depicted using the \textit{virdis} color palette. The presented algorithm took $0.00465$s per frame in a Python implementation. }
	\label{fig:gates-1-2}
	\vspace*{-0.7cm}
\end{figure}


\subsection{Background}
Most forecasting algorithm are well characterized by the underlying evolution model they adopt.
Such models come in a variety flavors, and are adapted to the task at hand (e.g. crowd modeling \cite{Helbing1992}).
This paper is focused on the construction of useful motion models for pedestrians that can aide the task of real-time forecasting for autonomous agents.
The simplest approach to forecasting with motion models forward integrates a Kalman filter based upon the observed heading.
Over short time scales this method may perform well, but the resulting distribution devolves into an imprecise Gaussian mass over longer time scales.
In particular, such models are less useful for forecasts beyond two seconds, especially when a pedestrian turns.
Nonetheless, these stochastic linear models serve as a reasonable default in the absence of any contextual knowledge.

More sophisticated models that attempt to leverage environmental data include Inverse Optimal Control (IOC) based models \cite{Ziebart2008,Ziebart2009,Kitani2012,Xie2013,Karasev2016}.
These IOC models have the desirable property of attributing intention and goal-seeking behavior to the agents.
For example, \cite{Kitani2012} extracts a Markov Decision Process (MDP) evolving on a finite 2D lattice by training on a small collection of features and trajectories.
The resulting MDP is light-weight since it is parametrized by only a small collection of coefficients equal in number to that of the feature maps.
Moreover, given an initial and final state, the probability distribution of the pedestrian at intermediate states is computable using matrix multiplication.
The speed of this computation, makes the algorithm of \cite{Kitani2012} a reasonable baseline for comparison for the algorithm that is presented in this paper.

This approach has been generalized in a variety of ways.
For example, time-dependent information, such as traffic signals, are incorporated in \cite{Karasev2016}, by relaxing the Markov property and considering a switched Markov process.
Other generalizations include replacing the finite-state space with a continuous one, and using a Markov jump process in the motion model.
Unfortunately the desired posteriors are difficult to compute in closed form, and as a result use sampling based methods.
The resulting accuracy of such methods, which can come at a large computational expense, can only be known in a probabilistic sense in that the error bounds are themselves random variables.

A limitation of IOC models occurs when there are locally optimal solutions between a start and end goals that can yield non-robust and imprecise behavior.
This occurs when agents make sharp turns due to intermediate criteria on the way toward reaching their final destination.
To address this, \cite{Ballan2016} adopt an empiricists approach, computing ``turning maps'' and attempting to infer how agents behave in a given patch.
The motion model is a Markov jump process and the relevant posteriors are approximated using sample based techniques similar to \cite{Karasev2016}.
The objective of \cite{Ballan2016} is not only prediction, but the development of a motion model learned on one scene that could then subsequently be transferred to other scenes. 
This requires representations of ``objects'' in the scene that do not depend rigidly on the finitely many labels an engineer managed to think of in a late-night brainstorming session.

Recent work has focused on constructing an unsupervised approach towards forecasting  \cite{Walker2014}.
Unlike all the approaches mentioned thus far, the agents in \cite{Walker2014} were not manually specified.
They were learned by detecting which sort of patches of video were likely to move, and how.
The resulting predictions outperformed \cite{Kitani2012} when comparing the most likely path with the ground truth using the Modified Hausdorff Distance.
As in all methods mentioned thus far, computational speed and accuracy of any predicted posteriors were not a concern, so no such results were reported.
However, since the motion model was a Markov jump process which required the application of a sample based technique, we should expect the same painful trade-off between error and speed to occur as in \cite{Karasev2016,Ballan2016}.

Many have approached pedestrian forecasting by deriving their motion model from interactions between pedestrians using physically motivated methods \cite{Helbing1995,Xu2012}. 
	Several models derive their motion models from \cite{Helbing1995} by incorporating collision avoidance through an interaction potential \cite{Pellegrini2009,Yamaguchi2011,Yi2016}. 
	However this method suffers from not planning for other pedestrian positions at future times. 
	Others take optical flow as input \cite{Hospedales2009,Wang2009,Emonet2011}.
	These approaches use variants of Hierarchical Dirichlet Processes on discretized optical flow to determine temporal motifs (i.e. classes of motion within the scene), or on a Markov model. 
	 These models are not agent based, and the lack of an explicit motion model limits their predictive power. 
	Recently, methods have been developed to predict trajectories by introducing and sampling Anticipatory Temporal Conditional Random Fields which incorporate learned affordances based on observed objectives within the scene \cite{Koppula2016}.
	 Others create agent-based models based on Gaussian Processes which perform poorly when trained on discretized trajectories \cite{Tay2008,Wang2008,Trautman2015}.
	Most recently, a method using Long Short-Term Memory (LSTM) was proposed to learn pedestrian motion models without making assumptions about the manner in which agents interact while having a rapid computation time \cite{Alahi2016}.

\subsection{Contributions}

The primary contributions of this paper are three-fold: first, an accurate motion model for pedestrian forecasting, second, an expedient method for computing approximations of relevant posteriors generated by our motion model, and finally hard error bounds on the proposed approximations.
The method proposed by this paper is able to work three times faster than the existing state of the art while improving upon its performance over long time horizons.
For clarification, we should mention that there are a number of things that we do not do.
For example, we do not concern ourselves with detection and tracking. 
Nor do we concern ourselves with updating our prediction as new data comes along.
We largely work in a 2D environment with a bird's eye view, operating under the assumption that the data has been appropriately transformed by a third party.
While it would be a straight forward extension to consider a first person perspective, it would detract from the presentation.

The rest of the paper is organized as follows: \S \ref{sec:model} describes our motion model as a Bayesian network, \S \ref{sec:efficient} describes how to compute probability densities for an agent's position efficiently, and \S \ref{sec:implementation} demonstrates the model by training and testing it on the Stanford drone dataset \cite{Robicquet2016}.

\section{Model}\label{sec:model}
This paper's goal is to generate a time-dependent probability density over $\mathbb{R}^2$, which predicts the true location of an agent in the future.
The input to the algorithm at runtime is a noisy measurement of position and velocity, $\hat{x}_0, \hat{v}_0 \in \mathbb{R}^2$.
If the (unknown) location of agent at time $t$ is given by $x_t \in \mathbb{R}^2$, then the distribution we seek is the posterior $\rho_t(x_t) := \Pr( x_t \mid \hat{x}_0, \hat{v}_0 )$ for each time $t \in \{\Delta t, \dots, N_t \Delta t \}$ for some user-specified $N_t \in \mathbb{N}$ and $\Delta t \in \mathbb{R}$.

To compute $\rho_t$, we build a probabilistic graphical model, which is composed of three parts:
\begin{enumerate}
	\item Reality: This is parametrized by the true position for all time, $x_t$, and the initial velocity of the agent $v_0$.
	\item Measurements:  This is represented by our sensor readings $\hat{x}_0$ and $\hat{v}_0$ and are independent of all other variables given the true initial position and velocity, $x_o, v_0$.
	\item Motion Model: This is represented by a trajectory $\check{x}_t$ and depends on a variety of other variables.
\end{enumerate}
We elaborate on these three components next.

\subsection{The Variables of the Model}
The model concerns the position of an agent $x_t \in \mathbb{R}^2$ for $t \in [0,N_t \Delta t]$.
We denote the position and velocity at time $t=0$ by $x_0$ and $v_0$ respectively.
At $t=0$, we obtain a measurement of position and velocity, denotes by $\hat{x}_0$ and $\hat{v}_0$.
Lastly, we have a variety of motion models, parametrized by a set $\mathcal{M}$ (described in the sequel).
For each model $m \in \mathcal{M}$, a trajectory $\check{x}_t$ given the initial position and velocity $x_0$ and $v_0$.
All these variables are probabilistically related to one another in a (sparse) Bayesian network, which we will describe next.

\subsection{The Sensor Model}
At time $t=0$, we obtain a noisy reading of position, $\hat{x}_0 \in \mathbb{R}^2$.
We assume that given the true position, $x_0 \in \mathbb{R}^2$, that the measurement $\hat{x}_0$ is independent of all other variables and the posterior $\Pr( \hat{x}_0 \mid x_0)$ is known.
We assume a similar measurement model for the measured initial velocity $\hat{v}_0$.

\subsection{The Agent Model}
All agents are initialized within some rectangular region $D \subset \mathbb{R}^2$.
We denote the true position of an agent by $x_t$.
We should never expect to know $x_t$ and the nature of its evolution precisely, and any model should account for its own (inevitable) imprecision.
We do this by fitting a deterministic model to the data and then smoothing the results.
Specifically, our motion model consists of a modeled trajectory $\check{x}_t$, which is probabilistically related to the true position by $x_t$ via a known and easily computable posterior, $\Pr(x_t \mid \check{x}_t)$.

Once initialized, agents come in two flavors: linear and nonlinear.
The linear agent model evolves according to the equation $\check{x}_t = x_0 + t v_0$ and so we have the posterior:
\begin{align}
	\Pr( \check{x}_t  \in A \mid x_0, v_0, lin) = \int_A \delta( \check{x}_t - x_0 - t v_0 ) d\check{x}_t.
\end{align}
for all measurable sets $A \subset \mathbb{R}^2$, where $\delta( \cdot )$ denotes the multivariate Dirac-delta distribution.
For the sake of convenience, from here on we drop the set $A$ and the integral when defining such posteriors since this equation is true for all measurable sets $A$.
We also assume the posteriors, $\Pr(x_0 \mid lin)$ and $\Pr( v_0 \mid lin, x_0)$ are known.

If the agent is of nonlinear type, then we assume the dynamics take the form:
\begin{align}n
	\frac{d \check{x}_t}{dt} = s \cdot X_k(\check{x}_t) \label{eq:ode}
\end{align}
where $X_k$ is a vector-field\footnote{A vector-field is an assignment of a velocity to each position in some space.  A vector-field on $\mathbb{R}^n$ is a map from $\mathbb{R}^n \to \mathbb{R}^n$.} drawn from a finite collection $\{X_1, \dots, X_n\}$, and $s \in \mathbb{R}$.
More specifically, we assume that each $X_k$ has the property that $\| X_k(x) \| = 1$ for all $x \in D$.
This property ensures that speed is constant in time.
As we describe in \S \ref{sec:implementation}, the stationary vector-fields $X_1,\dots,X_n$ are learned from the dataset.

It is assumed that $k$ and $s$ are both constant in time, so that $\check{x}_t$ is determined from the triple $(x_0,k,s)$ by integrating \eqref{eq:ode} with the initial condition $x_0$.
This insight allows us to use the motion model to generate the posterior for $\Pr( \check{x}_t \mid x_0, k, s)$.
For each initial condition, $x_0$, we can solve \eqref{eq:ode} as an initial value problem, to obtain a point $\check{x}_t$ with initial condition $\check{x}_0 = x_0$.
This process of solving the differential equation takes an initial condition, $\check{x}_0$, and outputs a final condition, $\check{x}_t$.
This constitutes a map which is termed the \emph{flow-map} \cite[Ch 4]{MTA}, and which we denote by $\Phi_{k,s}^t$.
Explicitly, we have the posterior:
\begin{align}
	\Pr( \check{x}_t \mid x_0 , k , s) = \delta( \check{x}_t - \Phi^{t}_{k,s}( x_0) )  d\check{x}_t \label{eq:x_check | ksx}
\end{align}
where $\Phi^{t}_{k,s}$ is the flow-map of the vector field $s \,X_k$ up to time $t$.
Note that this flow-map can be evaluated for an initial condition by just integrating the vector field from that initial condition.
Note the variables $k,s$ and $x_0$ determine $v_0$.
Thus we have the posterior:
\begin{align}
	\Pr( v_0 \mid k, s, x_0) = \delta( v_0 -s X_k( x_0) ) dv_0. \label{eq:v | ksx}
\end{align}
In summary, the agent models are parametrized by the set $\mathcal{M} = \{ lin \} \cup \left( \mathbb{R} \times \{ 1 , \dots, n \} \right)$ whose elements determine the type of agent motion.

\subsection{The Full Model}
Concatenating the measurement model with our motion models yields the Bayesian network, where $M \in \mathcal{M}$ denotes the model of the agent:
\begin{align}
\begin{tikzpicture}[thick, var/.style={circle,draw,thin,rounded corners,shade,top color=blue!50,minimum size = 2mm}]
	\node[var] (M) {$M$};
	\node[var] (x)[right=of M] {$x_0$};
	\node[var] (v)[below=of x] {$v_0$};
	\node[var] (x_hat) [right=of x] {$\hat{x}_0$};
	\node[var] (v_hat) [right=of v] {$\hat{v}_0$};
	\node[var] (x_check_t) [left=of v] {$\check{x}_t$};
	\node[var] (x_t) [left=of x_check_t] {$x_t$};
	\draw[->] (M) to (x);
	\draw[->] (M) to (v);
	\draw[->] (M) to (x_check_t);
	\draw[->] (x) to (x_hat);
	\draw[->] (x) to (x_check_t);
	\draw[->] (x) to (v);
	\draw[->] (v) to (x_check_t);
	\draw[->] (v) to (v_hat);
	\draw[->] (x_check_t) to (x_t); 
\end{tikzpicture}.\label{eq:pgm}
\end{align}
We use this Bayesian network to compute $\rho_t$.
In particular
\begin{align}
	&\rho_t(x_t ) := \Pr( x_t \mid \hat{x}_0, \hat{v}_0 ) \\
	&=\sum_{k} \int \Pr( x_t, k , s  \mid \hat{x}_0, \hat{v}_0 ) ds  + \Pr( x_t, lin \mid \hat{x}_0, \hat{v}_0 ). \label{eq:decomposition}
\end{align}
$\Pr( x_t, lin \mid \hat{x}_0, \hat{v}_0 )$ is expressible in closed form when the posteriors $\Pr( x_0 \mid lin)$ and $\Pr( v_0 \mid lin,x_0)$ are known. 
In this instance, the numerical computation of $\Pr(x_t, lin \mid \hat{x}_0, \hat{v}_0)$ poses a negligible burden and the primary computational burden derives from computing $\sum_{k} \int \Pr( x_t, k , s  \mid \hat{x}_0, \hat{v}_0 ) ds$.

\section{Efficient Probability Propagation} \label{sec:efficient}
This section details how the modeling of the agent's motion as satisfying an ODE can be leveraged to compute $\rho_t(x_t)$ quickly and accurately.
To begin, rather than focusing on computing $\rho_t(x_t)$, we describe how to compute the joint probability $\Pr( x_t , \hat{x}_0, \hat{v}_0)$.
We can obtain $\rho_t(x_t)$ by normalizing $\Pr( x_t , \hat{x}_0, \hat{v}_0)$ with respect to integration over $x_t$.
We can approximate the integration over $s$ in \eqref{eq:decomposition}, with a Riemann sum.

Let us assume that $\Pr(s \mid k)$ is compactly supported for all $k = 1,\dots,n$, and the supported is always contained
in some interval $[ - \bar{s} , \bar{s} ]$ for some $\bar{s} > 0$.
Given a regular partition $\{ s_0, s_1, \dots, s_n \}$ of step-size $\Delta s > 0$ on $[-\bar{s}, \bar{s}]$,  we can conclude that the integral term in \eqref{eq:decomposition} is approximated by
\begin{align}
	\begin{split}
	&\sum_{k} \int \Pr( x_t, k , s , \hat{x}_0, \hat{v}_0 ) ds = \\
	 &\underbrace{\Delta s \sum_{j} \sum_{k} \Pr( x_t, k , s_j , \hat{x}_0, \hat{v}_0)}_{\text{approximation}}
	  +
	  \underbrace{\varepsilon_s}_{\text{error}}
	 \end{split} \label{eq:approximation 0}
\end{align}
where the error is bounded by $\int| \varepsilon_s | ds \leq  TV(x_t, \hat{x}_0, \hat{v}_0) \Delta s$
where $TV(x_t, \hat{x}_0, \hat{v}_0)$ is the sum, with respect to $k$, of the total variation of $\Pr( x_t, k , s, \hat{x}_0, \hat{v}_0 )$ with respect to $s$ for fixed $x_t, \hat{x}_0, \hat{v}_0$.
Since this error term can be controlled, the problem of solving $\rho_t(x_t)$ is reduced to that of efficiently computing $ \Pr( x_t, k , s_j, \hat{x}_0, \hat{v}_0)$
for a fixed collection of $s_j$'s.

$\hat{x}_0$ and $\hat{v}_0$ are measured and are assumed fixed for the remainder of this section.
To begin, from \eqref{eq:pgm}  notice that:
\begin{align}
	\Pr( x_t, k,s,\hat{x}_0, \hat{v}_0) &= \int \Pr( x_t, \check{x}_t , \hat{x}_0, \hat{v}_0, k,s) d\check{x}_t  \label{eq:convolve} \\
	&= \int \Pr( x_t \mid \check{x}_t ) \Pr(\check{x}_t , \hat{x}_0, \hat{v}_0, k,s) d\check{x}_t \nonumber
\end{align}
Observe that from the last line that $\Pr( x_t, k,s, \hat{x}_0, \hat{v}_0)$ is a convolution of the joint distribution $\Pr( \check{x}_t , \hat{x}_0, \hat{v}_0, k,s)$.
Assuming, for the moment, that such a convolution can be performed efficiently, we focus on computation of $\Pr( \check{x}_t , \hat{x}_0, \hat{v}_0, k,s)$.
Again, \eqref{eq:pgm} implies:
\begin{align}
	&\Pr( \check{x}_t , \hat{x}_0, \hat{v}_0, k,s) =\int \Pr( \check{x}_t , x_0, \hat{x}_0, v_0, \hat{v}_0, k,s) dx_0 \, dv_0 \nonumber \\
	&= \int \Pr( \check{x}_t \mid  x_0, k,s, v_0) \Pr( \hat{v}_0 \mid v_0) \cdot \\
	& \hspace{10ex}  \cdot \Pr( v_0 \mid k,s,x_0) \Pr(\hat{x}_0, x_0, k, s) dx_0 \, dv_0 \nonumber \\
		&= \int \delta\left( \check{x}_t - \Phi_{k,s}^{t}( x_0) \right) \delta\left( v_0 - s X_k (x_0) \right) \cdot \\
		&\hspace{10ex} \cdot\Pr( \hat{v}_0 \mid v_0) \Pr(\hat{x}_0, x_0,k, s) dx_0\, dv_0, \nonumber
\end{align}
where the last equality follows from substituting \eqref{eq:x_check | ksx} and \eqref{eq:v | ksx}.
 Carrying out the integration over $v_0$ we observe:
\begin{align}
\begin{split}
	&\Pr( \check{x}_t , \hat{x}_0, \hat{v}_0, k,s) = \int \delta\left( \check{x}_t - \Phi_{k,s}^{t}( x_0) \right)  \cdot \\
	&\hspace{10ex} \cdot \Pr(\hat{x}_0, x_0, k, s) \Psi(\hat{v}_0 ;k,s,x_0) dx_0,
\end{split}
\label{eq:push forward}
\end{align}
 where $\Psi( \hat{v}_0 ;k,s,x_0) := \left. \Pr( \hat{v}_0 \mid v_0) \right|_{v_0 = s X_k(x_0)}$.
 We may approximate $\Pr(\hat{x}_0, x_0, k, s) \Psi( \hat{v}_0 ; k, s, x_0)$ as a sum of weighted Dirac-delta distributions supported on a regular grid, since $\Pr(\hat{x}_0, x_0, k, s) \Psi( \hat{v}_0 ; k, s, x_0)$  is of bounded variation in the variable $x_0$ (with all other variables held fixed).
 
 To accomplish this, let $S_L(\hat{x}_0)$ denote the square of side length $L>0$ centered around $\hat{x}_0$.
 Choose  $L>0$ to be such that $\int_{S_L(\hat{x}_0)} \Pr( x_0 \mid \hat{x}_0) dx_0 = 1 - \varepsilon_{tol}$ for some error tolerance $\varepsilon_{tol}>0$.
 Then, for a given resolution $N_x \in \mathbb{N}$ define the regular grid on $S_L(\hat{x}_0)$ as $\Gamma_L( \hat{x}_0 ; N_x ) := \left\{ x_0^{i,j} \mid i,j \in \{ -N_x,\dots,N_x \} \right\}$, where $x_0^{i,j} = \hat{x}_0 + \frac{L}{2N_x}(i,j)$.
 The grid spacing is given by $\Delta x = ( \frac{L}{2N_x}, \frac{L}{2N_x} )$.
We approximate the smooth distribution $\Pr(\hat{x}_0, x_0, k, s) \Psi( \hat{v}_0 ; k, s, x_0 )$ as a weighted sum of Dirac-deltas (in the variable $x_0$) supported on $\Gamma_L( \hat{x}_0;N)$:
\begin{align}
		&\Pr(\hat{x}_0, x_0, k, s) \Psi( \hat{v}_0 ; k, s, x_0 ) = \nonumber \\
		& \underbrace{\left( \sum_{i,j=-N}^{N} W(k,s,i,j,\hat{x}_0) \delta( x_0 - x_0^{i,j} ) \right)}_{\text{approximation}}
		+ \underbrace{ \varepsilon_0(x_0) }_{\text{error}}
	\label{eq:approximation 1}
\end{align}
where $W(k,s,i,j,\hat{x}_0)$ is the evaluation of $\Pr(\hat{x}_0, x_0, k, s) \Psi( \hat{v}_0 ; k, s, x_0 )$ at the grid point $x_0 = x_0^{i,j} \in \Gamma( \hat{x}_0; N)$.
More explicitly, this evaluation can be done for each grid point by using only the assumed posterior models in \eqref{eq:pgm}.
For fixed $k$ and $s$, the expression $\Pr(\hat{x}_0, x_0, k, s) \Psi( \hat{v}_0 ; k, s, x_0 )$ is a density in $x_0$ and the error term in \eqref{eq:approximation 1} has a magnitude of $\| \varepsilon_0 \|_{L^1} \sim \mathcal{O}( | \Delta x |  + \varepsilon_{tol} )$ with respect to the $L^1$-norm in $x_0$.

Substitution of \eqref{eq:approximation 1} into the final line of \eqref{eq:push forward} yields:
\begin{align}
	\begin{split}
	&\Pr( \check{x}_t , \hat{x}_0, \hat{v}_0, k,s) = \\
	&\quad \sum_{i,j} W(k,s,i,j,\hat{x}_0) \delta \left( \check{x}_t - \Phi_{k,s}( x_0^{i,j}) \right) + \varepsilon_t( \check{x}_t)
	\end{split}
	\label{eq:makesense}
\end{align}
where $\varepsilon_t( \check{x}_t) = \int \delta\left( \check{x}_t - \Phi_{k,s}^{t}( x_0) \right)  \varepsilon_0(x_0) dx_0$.
The first term of the right hand side of \eqref{eq:makesense} is computable by flowing the points of the grid, $\Gamma_L(\hat{x}_0; N_x)$, along the vector field $s X_k$.
The second term, $\varepsilon_t$, may be viewed as an error term.
In fact, this method of approximating $\Pr( \check{x}_t , \hat{x}_0, \hat{v}_0, k,s)$ as a sum of Dirac-delta distributions is adaptive, in that the error term does not grow in total mass, which is remarkable since many methods for linear evolution equations accumulate error exponentially in time \cite{Leveque1992,Gottlieb2001}:

\begin{thm} \label{thm:error}
	The error term, $\varepsilon_t \sim \mathcal{O}( | \Delta x | + \varepsilon_{tol} )$ in the $L^1$-norm, for fixed $k,s,\hat{x}_0$, and $\hat{v}_0$.
	Moreover, $\| \varepsilon_t \|_{L^1}$ is constant in time.
\end{thm}
\begin{proof}
	To declutter notation, let us temporarily denote $\Phi_{k,s}^t$ by $\Phi$.
	We observe
\begin{align*}
	\| \varepsilon_t \|_{L^1} &= \int \left| \int \delta( \check{x}_t - \Phi(x_0) ) \varepsilon_0(x_0) dx_0 \right| d\check{x}_t \\
	&= \int \det( \left. D\Phi \right|_{\Phi^{-1}( \check{x}_t) }) |\varepsilon_{0}( \Phi^{-1}( \check{x}_t )) | d\check{x}_t \\
	&= \int | \varepsilon_0( u) | du = \| \varepsilon_0 \|_{L^1}
\end{align*}
As $\varepsilon_0$ is of magnitude $\mathcal{O}( |\Delta x| + \varepsilon_{tol} )$ the result follows.
\end{proof}

While this allows us to compute posteriors over the output of our models, $\check{x}_t$, we ultimately care about densities over the true position.
The following corollary of Theorem \ref{thm:error} addresses this:
\begin{cor} \label{cor:error}
	The density
	\begin{align}
		\sum_{i,j} W(k,s,i,j,\hat{x}_0) \left. \Pr( x_t \mid \check{x}_t ) \right|_{ \check{x}_t = \Phi_{k,s}^t( x_0^\alpha) } \label{eq:approximation 2}
	\end{align}
	is an approximation of $\Pr( x_t, k, s, \hat{x}_0, \hat{v}_0)$ with a constant in time error bound of magnitude $\mathcal{O}( |\Delta x| + \varepsilon_{tol} )$.
\end{cor}
\begin{proof}
	By \eqref{eq:convolve}
	\begin{align*}
		\Pr(x_t, k, s, \hat{x}_0, \hat{v}_0) = \int \Pr( x_t \mid \check{x}_t ) \Pr( \check{x}_t, k,s,\hat{x}_0, \hat{v}_0) d\check{x}_t
	\end{align*}
	Substitution of \eqref{eq:approximation 1} yields
	\begin{align*}
		\begin{split}
		&\Pr(x_t, k, s, \hat{x}_0, \hat{v}_0) \nonumber \\
		& = \sum_{i,j} W(k,s,i,j,\hat{x}_0) \left. \Pr( x_t \mid \check{x}_t ) \right|_{ \check{x}_t = \Phi_{k,s}^t( x_0^{i,j}) } + \tilde{\varepsilon}_t(x_t)
		\end{split}
	\end{align*}
	where the error term is
	\begin{align}
		\tilde{\varepsilon}_t(x_t) = \int \Pr( x_t \mid \check{x}_t) \varepsilon_t( \check{x}_t) d \check{x}_t
	\end{align}
	and $\varepsilon_t$ is the error term of \eqref{eq:approximation 1}.
	We see that the $L^1$-norm of $\tilde{\varepsilon}_t$ is 
	\begin{align}
		\| \tilde{\varepsilon}_t \|_{L^1} &= \int \left| \int \Pr( x_t \mid \check{x}_t) \varepsilon_t( \check{x}_t) d \check{x}_t \right| dx_t \\
			&\leq \int \Pr( x_t \mid \check{x}_t ) | \varepsilon_t |( \check{x}_t) d\check{x}_t \, dx_t
	\end{align}
	Implementing the integration over $x_t$ first yields:
	\begin{align}
		\| \tilde{\varepsilon}_t \|_{L^1} \leq \int | \varepsilon_t |( \check{x}_t) d\check{x}_t =: \| \varepsilon_t \|_{L^1}
	\end{align}
	which is $\mathcal{O}( | \Delta x | + \varepsilon_{tol} )$ by Theorem \ref{thm:error}.
\end{proof}

Corollary \ref{cor:error} justifies using \eqref{eq:approximation 2} as an approximation of $\Pr( x_t, k,s,\hat{x}_0, \hat{v}_0)$.
This reduces the problem of computing $\rho_t(x_t)$ to the problem of computing the weights $W(k,s,i,j,\hat{x}_0)$ and the points $\Phi_{k,s}^t(x_0^{i,j})$ for all $k,s$ and points $x_0^{i,j} \in \Gamma_L( \hat{x}_0; N_x)$.
We can reduce this burden further by exploiting the following symmetry:
\begin{thm} \label{thm:symmetry}
	$\Phi_{k,s}^t = \Phi_{k,1}^{st}$.
\end{thm}
\begin{proof}
	Say $x(t)$ satisfies the ordinary differential equation $x'(t) = sX_k(x(t))$ with the initial condition $x_0$.
	In other words, $x(t) = \Phi_{k,s}^{t}(x_0)$.
	Taking a derivative of $x(t/s)$, we see $\frac{d}{dt} (x(t/s)) = x'(t/s) /s = X_k(x(st))$.
	Therefore $x(t/s) = \Phi_{k,1}^{t}( x_0)$.
	Substitution of $t$ with $\tau = t/s$ yields $x(\tau) = \Phi_{k,1}^{s \tau} (x_0)$.
	As $x(\tau) = \Phi_{k,s}^{\tau}(x_0)$ as well, the result follows.
\end{proof}
This, allows us to compute $\Phi_{k,s}^t( x_0^{\alpha})$ using computations of $\Phi_{k,1}^t(x_0^{\alpha})$, which yields the following result:
\begin{thm} \label{thm:main}
	Let $\{s_1,\dots,s_n\}$ be a regular partition on the support of $\Pr(s)$.
	Then the density
	\begin{align}
		\begin{split}
		&\Delta s \sum_{i,j,k,m} W(k,s_m,i,j,\hat{x}_0) \left. \Pr( x_t \mid \check{x}_t ) \right|_{ \check{x}_t = \Phi_{k,1}^{s_m t}( x_0^{i,j}) } \\
		&+\Pr( x_t , lin, \hat{x}_0, \hat{v}_0 )
		\end{split}
		\label{eq:approximation 3}
	\end{align}
	approximates $\Pr( x_t, \hat{x}_0, \hat{v}_0)$ with an error of size $\mathcal{O}( \Delta s + \Delta x + \varepsilon_{\rm tol})$.
\end{thm}
\begin{proof}
	Substitute Theorem \ref{thm:symmetry} into Corollary \ref{cor:error}, to replace $\Phi_{k,s}^t$ with $\Phi_{k,1}^{st}$.
	This gives us an error term of size $\Delta x$, if we compute the integral over $s$ exactly.
	Using \eqref{eq:approximation 0}, we can compute the integral over $s$ approximately, with an error of magnitude $\Delta s$.
\end{proof}

This is a powerful result since it allows us to compute $\Pr( x_t, \hat{x}_0, \hat{v}_0)$  (and thus $\rho_t(x_t)$) at times $t \in \{\Delta t, \dots, N_t \Delta t\}$ with a single computation of $\Phi_{k,1}^{\ell \Delta t \bar{s}}( x_0^{i,j})$ for each $\ell \in  \{-N_t,\dots,N_t\}$, $k \in \{1,\dots,n\}$, and $x_0^{i,j} \in \Gamma_L( \hat{x}_0 , N_x)$.
To appreciate this, assume we are given $\Phi_{k,1}^{m  \bar{s} \Delta t}(x_0^{i,j})$ for all $m \in \{ -\ell, \ldots, \ell \}$.
We can then use the regular partition
\begin{align}
	\mathcal{P}_\ell := \left\{ \frac{m}{\ell} \bar{s} \right\}_{m=-\ell}^{\ell} \label{eq:partition}
\end{align}
of $[-\bar{s}, \bar{s}]$ to compute a Riemann sum approximation using Theorem \ref{thm:main}.
The partition $\mathcal{P}_\ell$ has a width of size, $\Delta s = \bar{s} / \ell$, and substitution in \eqref{eq:approximation 3} yields an approximation of $\Pr( x_t, \hat{x}_0, \hat{v}_0)$ at time $t=\ell \Delta t$
given by
\begin{align}
		&\frac{\bar{s}}{\ell} \sum_{i,j,k,m} W(k,m \bar{s} / \ell ,i,j,\hat{x}_0) \left. \Pr( x_t \mid \check{x}_t ) \right|_{ \check{x}_t = \Phi_{k,1}^{m \bar{s} \Delta t}( x_0^{i,j}) }  \nonumber \\
		&\qquad+\Pr( x_t , lin, \hat{x}_0, \hat{v}_0 ).
\end{align}
As we have already computed $\Phi_{k,1}^{m \bar{s} \Delta t}$ for $| m | \leq \ell$ by assumption, the only obstacle to computing this sum is the computation of the weights $W(k,m \bar{s} / \ell ,i,j,\hat{x}_0)$.
If we want to compute $\Pr( x_t, \hat{x}_0, \hat{v}_0)$ at $t=(\ell + 1 ) \Delta t$, then by reusing the earlier computation one only need to compute $\Phi_{k,1}^{m  \bar{s} \Delta t}(x_0^{i,j})$ for $m \in \{ -( \ell + 1 ), (\ell + 1 )\}$ along with the weights to obtain an approximation using the partition $\mathcal{P}_{\ell + 1}$.


The procedure to compute  $\rho_t(x_t)$ is summarized in Algorithm \ref{alg:1}.
For fixed $k$, $i$, and $j$, the computation of $\Phi_{k,1}^t(x_0^{i,j})$ at each $t = \{-N_t \Delta t \bar{s} ,\dots, N_t \Delta t \bar{s}\}$ takes $\mathcal{O}(N_t)$ time using an explicit finite difference scheme and can be done in parallel for each $k \in \{1,\dots, n\}$ and $x_0^{i,j} \in \Gamma_L( \hat{x}_0; N_x)$.
Similarly, computing $W(k,s,i,j,\hat{x}_0)$ constitutes a series of parallel function evaluations over tuples $(k,s,i,j)$ of the posterior distributions described in \eqref{eq:pgm}, where the continuous variable $s$ is only required at finitely many places in Algorithm \ref{alg:1}.
If the posteriors represented by the arrows in \eqref{eq:pgm} are efficiently computably then the computation of $W(k,s,i,j, \hat{x}_0)$ is equally efficient.

\begin{algorithm}[t]
		  \caption{Algorithm to compute  $\rho_t$ for each $t \in  \{\Delta t, 2 \Delta t, \dots, N_t \Delta t$\}.}
		  \label{alg:1}
		  \begin{algorithmic}[1]
		    \Require $\bar{s} > 0$, $N_t \in \mathbb{N}$, $\Delta t > 0$,  $ \{ X_k \}_{k=1}^{n}$, $\hat{x}_0$, $\hat{v}_0$, $\Gamma_L( \hat{x}_0, N_x)$, $\Pr(x_0 \mid M), \Pr(x_0 \mid \hat{x}_0), \Pr(v_0 \mid \hat{v}_0 ), \Pr(M \in \mathcal{M}), \text{ and } \Pr(s \mid k)$.
		    \For {$\ell \in \{1,\dots,N_t\}$}
		    \State Compute $\Phi_{k,1}^{ \pm \bar{s} \ell \Delta t}( x_0^{i,j} )$ for all $x_0^{i,j} \in \Gamma_L( \hat{x}_0 ; N_x)$ and \hspace*{12pt} for all $k \in  \{1,\dots, n\}$.
		    \For {$m \in \{-\ell,\dots,\ell\}$}
		    \State Compute $W(k,s,i,j,\hat{x}_0)$ for $s = \bar{s} m / \ell$ for all \hspace*{28pt} $i,j\in \{-N_x,\dots, N_x\}$ and $k\in \{1,\dots,n\}$.
		    \EndFor  
		    \State Compute $\Pr(\check{x}_t, \hat{x}_0, \hat{v}_0)$ at $t=\ell \Delta t$ via Theorem \ref{thm:main} 
		    \hspace*{12pt} with partition $\mathcal{P}_\ell$ of \eqref{eq:partition}.\label{alg1:step2}
		    \State Apply Bayes' Theorem to obtain $\Pr(\check{x}_t \mid \hat{x}_0, \hat{v}_0)$.
		    \State Compute $\rho_t(x_t)$ at $t= \ell \cdot \Delta t$ via \eqref{eq:convolve}.
		    \EndFor
		  \end{algorithmic}
\end{algorithm}

\section{Implementation and Experimental results} \label{sec:implementation}
  Given the model established in the previous section, we describe an implementation to showcase one way the model can be applied to observational data.
  In particular, we must learn the vector fields $\{X_1, \dots, X_n\}$, the posteriors $\Pr( x_0 \mid M)$ and the priors $\Pr(M)$ for $M \in \mathcal{M}$ from the data.
  For the purposes of demonstration, we use the Stanford Drone Dataset \cite{Robicquet2016}.
  More generally, we assume that for a fixed scene we have a database of previously observed trajectories $\{ \hat{x}^1, \dots, \hat{x}^N\}$.
  From this data we tune the parameters of the model ($\{X_1, \dots, X_n\}$, $\Pr( x_0 \mid M)$ and $\Pr(M)$)  appropriately.
  
  \subsection{Learning the Vector Fields}
 We begin by identifying the number of possible vector-fields.
To do, this we use a clustering algorithm on the trajectories using Affinity Propagation \cite{FreyDueck2007} and a custom distance measure defined on the endpoints of trajectories. 
	Let one trajectory $A$ start at $(x_{A,\mathrm{start}}, y_{A,\mathrm{start}})$ and end at  $(x_{A,\mathrm{end}}, y_{A,\mathrm{end}})$, and another trajectory $B$ start at $(x_{B,\mathrm{start}}, y_{B,\mathrm{start}})$ and end at $(x_{B,\mathrm{end}}, y_{B,\mathrm{end}})$ . 
	We define the points $\mathbf{a}_1 = (x_{A, \mathrm{start}}, y_{A, \mathrm{start}}, x_{A, \mathrm{end}}, y_{A, \mathrm{end}})$, $\mathbf{a}_2 = (x_{A, \mathrm{end}}, y_{A, \mathrm{end}}, x_{A, \mathrm{start}}, y_{A, \mathrm{start}})$, and $\mathbf{b} = (x_{B, \mathrm{start}}, y_{B, \mathrm{start}}, x_{B, \mathrm{end}}, y_{B, \mathrm{end}})$. 
	We then define our distance measure as $d(A, B) :=  \min \left\{ d_e(\mathbf{a}_1, \mathbf{b}), d_e(\mathbf{a}_2, \mathbf{b}) \right\}$, for the euclidean distance $d_e$ in $\mathbb{R}^4$. 
	This function measures the distance between the endpoints irrespective of their ordering, which means that a trajectory that starts from point A and ends at point B will be close to a trajectory that starts from point B and ends at point A.
	The scale of the datasets we tested on had large enough spatial scale that clustering based on endpoints captured people moving from destination to destination, e.g. from a storefront to the sidewalk at the edge of a scene. 
	 On this dataset, distance functions that utilize the entire trajectory did not identify pedestrian intent as well as our method \cite{Morris2009,Lee2007}. It appears the metrics proposed in \cite{Morris2009} put trajectories that were similar for periods of time together even though they had different intents.
	This clustering of the end-points induces a clustering of the trajectories.
  Suppose we obtain clusters $S_1, \dots, S_n$ consisting of trajectories from our data set, as well as a set of unclassified trajectories, $S_0$.
  
  For each set $S_k$ we learn a vector-field that is approximately compatible with that set.
  Since most trajectories appearing in the dataset have roughly constant speed, we chose a vector-field that has unit magnitude. 
  That is, we assume the vector-field takes the form $X_k(x) = \left( \cos( \Theta_k(x) ) , \sin(\Theta_k(x)) \right)$ for some scalar function $\Theta_k(x)$.
  Learning the vector-fields then boils down to learning the scalar function $\Theta_k$.
  We assume $\Theta_k = \sum_{\alpha} \theta_{k,\alpha} L_{\alpha}(x)$ for some collection of coefficients, $\theta_{k,\alpha}$, and a fixed collection of basis functions, $L_{\alpha}$.
  We choose $L_{\alpha}$ to be a set of low degree Legendre polynomials.
 
 $\Theta_k$ is learned by computing the velocities observed in the cluster, $S_k$.
  These velocities are obtained by a low order finite difference formula.
  Upon normalizing the velocities, we obtain a unit-length velocity vectors, $v_{i,k}$, anchored at each point, $x_{i,k}$, of $S_k$.
  We learn $\Theta_k$ by defining the cost-function:
  \begin{align}
  	C[ \Theta_k] = \sum_i \langle v_{i,k} , ( \cos(\Theta_k( x_{i,k}) , \sin( \Theta_k( x_{i,k} ) ) \rangle
  \end{align}
  which penalizes $\Theta_k$ for producing a misalignment with the observed velocities at the observed points of $S_k$.
  When $\Theta_{k}$ includes high order polynomials (e.g. beyond 5th order), we also include a regularization term to bias the cost towards smoother outputs.
  Using the $H^1$-norm times a fixed scalar suffices as a regularization term. 
  
  \subsection{Learning $\Pr( x_0 \mid M)$ and $\Pr(M)$}
  
We first considering the nonlinear models.
  We begin by assuming that $x_0$ is independent of $s$ given $k$, i.e. $\Pr( x_0 \mid k,s) = \Pr(x_0 \mid k)$.
  Additionally, we assume that $s$ and $k$ are independent.
  This means that we only need to learn $\Pr( x_0 \mid k)$, $\Pr(k)$, and $\Pr(s)$.
  We let $\Pr(k) = (n+1)^{-1}$ and $\Pr(s) \sim \mathcal{U}( [-s_{\max}, s_{\max} ] )$  where $s_{\max}>0$ is the largest observed speed in the dataset.  
  This implies that $\Pr(lin) = (n+1)^{-1}$ as well.
  
  For each $k$ we assume $\Pr( x_0 \mid k) = \frac{1}{Z_k} \exp( - V_k(x_0) )$ and $V_k$ is a function whose constant term is $0$ and is  $V_k(x_0; \mathbf{c} ) := \sum_{|\alpha|< d} c_{\alpha} L_{\alpha}( x_0)$ for a collection of basis functions, $L_{\alpha}$ and coefficients $\mathbf{c} = \{ c_{\alpha} \}_{|\alpha| < d}$.
  We chose our basis functions to be the collection of tensor products of the first six Legendre polynomials, normalized to the size of the domain.
  Then, one may fit the coefficients $c_{\alpha}$ to the data by using a log-likelihood criterion.
  The resulting (convex) optimization problem takes the form:
  \begin{align}
  	\mathbf{c}^* = \inf_{ \mathbf{c} } \sum_{x \in S_k} V_k( x_0; \mathbf{c})
  \end{align}
  Where the norm on $\mathbf{c}$ is a sup-norm.
  We bias this optimization towards smooth functions by adding a penalty to the cost function.
    Finally, we let $\Pr( x_0 \mid lin) \sim \mathcal{U}(D)$.
  
  \subsection{Learning the Measurement Model}
  We assume a Gaussian noise model (i.e. $\Pr( \hat{x}_0 \mid x_0 ) \sim \mathcal{N}( x_0 , \sigma_x)$ and $\Pr( \hat{v}_0 \mid v_0 ) \sim \mathcal{N}( v_0, \sigma_v)$).
  Therefore, our model is parametrized by the standard deviations $\sigma_x$ and $\sigma_v$.
  We assume that the true trajectory of an agent is smooth compared to the noisy output of our measurement device.
  This justifies smoothing the trajectories, and using the difference between the smoothed signals and the raw data to learn the variance $\sigma_x$.
  To obtain the results in this paper we have used a moving average of four time steps (this is $0.13$ seconds in realtime).
  We set $\sigma_v = 2 \sigma_x / \Delta t$ where $\Delta t > 0$ is the time-step size.  
  This choice is justified from the our use of finite differences to estimate velocity.
  In particular, if velocity is approximated via finite differencing as $v(t) \approx (x(t+h) - x(t))\,\Delta t^{-1} + \mathcal{O}(h)$ and the measurements are corrupted by Gaussian noise, then the measurement $\hat{v}(t)$ is related to $v(t)$ by Gaussian noise with roughly the same standard deviation as $(x(t+h) - x(t))\,\Delta t^{-1}$.

  \begin{figure*}[t!]
  \begin{subfigure}[t]{0.48\textwidth}
	\centering
	\begin{minipage}[c]{0.47cm}
		\rotatebox{90}{\small{Our Algorithm}   }
	\end{minipage}
	\begin{minipage}[c]{0.3\linewidth}
		\includegraphics[width=\linewidth]{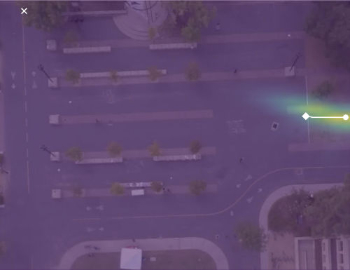}
	\end{minipage}
	\begin{minipage}[c]{0.3\linewidth}
		\includegraphics[width=\linewidth]{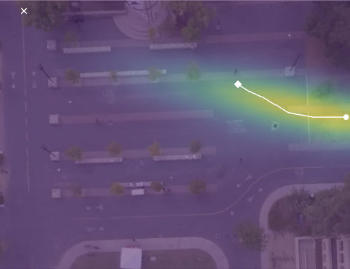}
	\end{minipage}
	\begin{minipage}[c]{0.3\linewidth}
		\includegraphics[width=\linewidth]{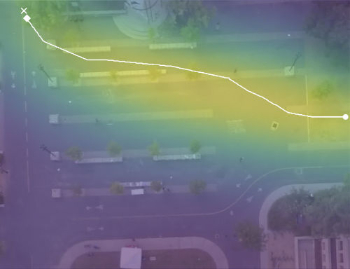}
	\end{minipage}
	
	\vspace{0.1cm}
	\begin{minipage}[c]{0.47cm}
		\rotatebox{90}{\small{Kitani et al.}  }
	\end{minipage}
	\begin{minipage}[c]{0.3\linewidth}
		\includegraphics[width=\linewidth]{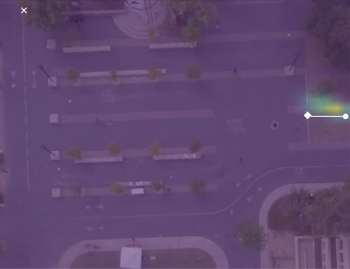}
	\end{minipage}
	\begin{minipage}[c]{0.3\linewidth}
		\includegraphics[width=\linewidth]{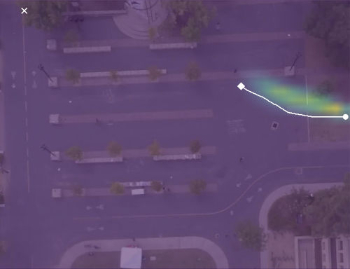}
	\end{minipage}
	\begin{minipage}[c]{0.3\linewidth}
		\includegraphics[width=\linewidth]{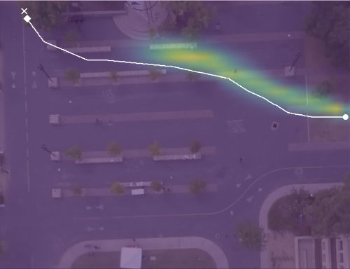}
	\end{minipage}
	
	\vspace{0.1cm}
		\begin{minipage}[c]{0.47cm}
			\rotatebox{90}{\small{S-LSTM} }  
		\end{minipage}
		\begin{minipage}[c]{0.3\linewidth}
			\includegraphics[width=\linewidth]{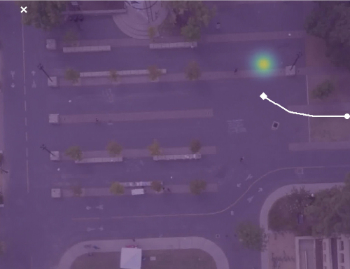}
		\end{minipage}
		\begin{minipage}[c]{0.3\linewidth}
			\includegraphics[width=\linewidth]{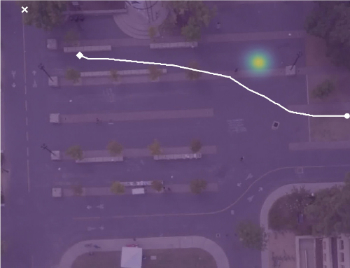}
		\end{minipage}
		\begin{minipage}[c]{0.3\linewidth}
			\includegraphics[width=\linewidth]{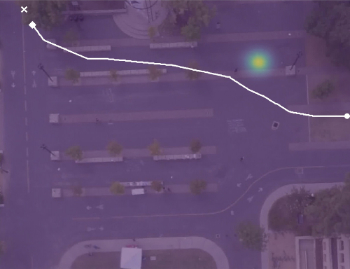}
		\end{minipage}
	
	\vspace{0.1cm}
	\begin{minipage}[c]{0.47cm}
		\rotatebox{90}{\small{Random Walk} }  
	\end{minipage}
	\begin{minipage}[c]{0.3\linewidth}
		\includegraphics[width=\linewidth]{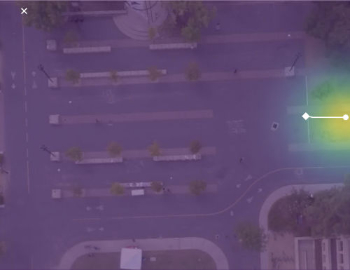}
	\end{minipage}
	\begin{minipage}[c]{0.3\linewidth}
		\includegraphics[width=\linewidth]{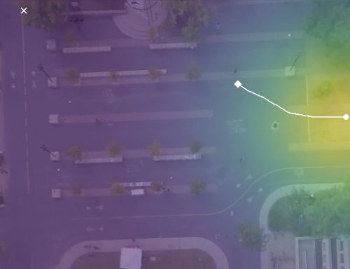}
	\end{minipage}
	\begin{minipage}[c]{0.3\linewidth}
		\includegraphics[width=\linewidth]{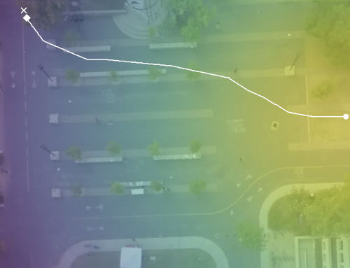}
	\end{minipage}
	
	\vspace{0.2cm}
	\hspace{.5cm}
	\begin{minipage}{0.3\linewidth}
		\centering
		$t = 1.66s$
	\end{minipage}
	\begin{minipage}{0.3\linewidth}
		\centering
		$t = 4.33s$
	\end{minipage}
	\begin{minipage}{0.3\linewidth}
		\centering
		$t = 12.33s$
	\end{minipage}
	\caption{Scene 1}
	\label{fig:bookstore-1-2}
	\end{subfigure} 
	\hspace*{0.5cm}
	\begin{subfigure}[t]{0.48\textwidth}
	\centering
	\begin{minipage}[c]{0.47cm}
		\rotatebox{90}{\small{Our Algorithm}}   
	\end{minipage}
	\begin{minipage}[c]{0.3\linewidth}
		\includegraphics[width=\linewidth]{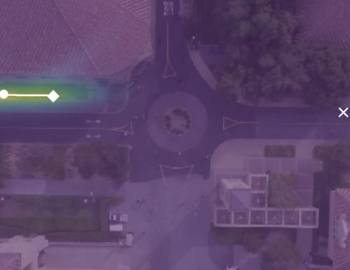}
	\end{minipage}
	\begin{minipage}[c]{0.3\linewidth}
		\includegraphics[width=\linewidth]{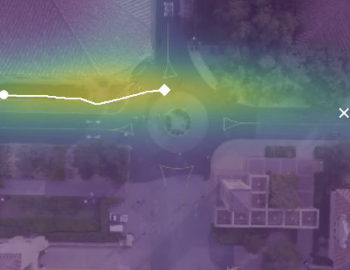}
	\end{minipage}
	\begin{minipage}[c]{0.3\linewidth}
		\includegraphics[width=\linewidth]{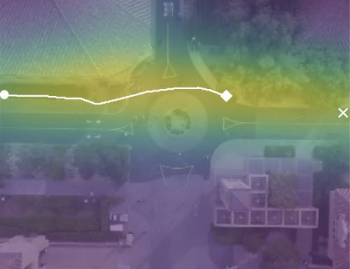}
	\end{minipage}
	
	\vspace{0.1cm}
	\begin{minipage}[c]{0.47cm}
		\rotatebox{90}{\small{Kitani et al.}} 
	\end{minipage}
	\begin{minipage}[c]{0.3\linewidth}
		\includegraphics[width=\linewidth]{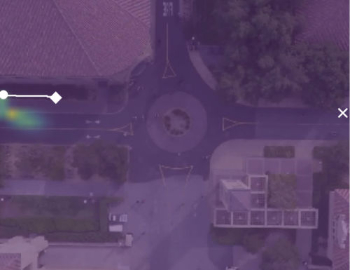}
	\end{minipage}
	\begin{minipage}[c]{0.3\linewidth}
		\includegraphics[width=\linewidth]{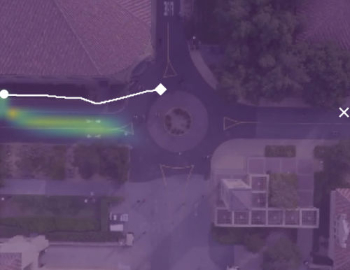}
	\end{minipage}
	\begin{minipage}[c]{0.3\linewidth}
		\includegraphics[width=\linewidth]{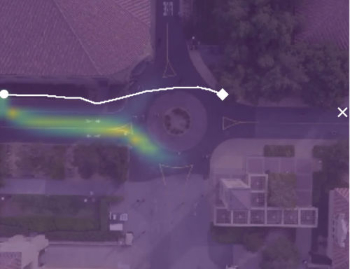}
	\end{minipage}
	
	\vspace{0.1cm}
	\begin{minipage}[c]{0.47cm}
		\rotatebox{90}{\small{S-LSTM}} 
	\end{minipage}
	\begin{minipage}[c]{0.3\linewidth}
		\includegraphics[width=\linewidth]{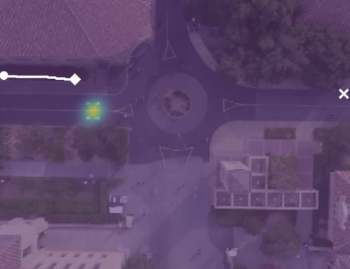}
	\end{minipage}
	\begin{minipage}[c]{0.3\linewidth}
		\includegraphics[width=\linewidth]{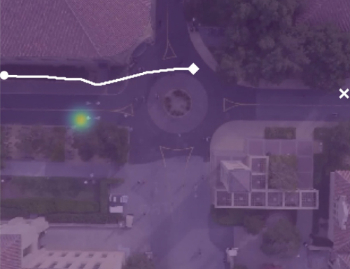}
	\end{minipage}
	\begin{minipage}[c]{0.3\linewidth}
		\includegraphics[width=\linewidth]{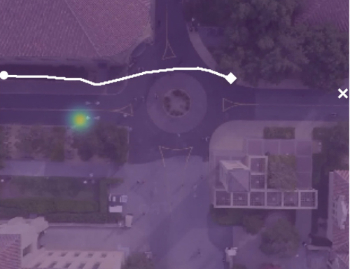}
	\end{minipage}
	
	\vspace{0.1cm}
	\begin{minipage}[c]{0.47cm}
		\rotatebox{90}{\small{Random Walk} }
	\end{minipage}
	\begin{minipage}[c]{0.3\linewidth}
		\includegraphics[width=\linewidth]{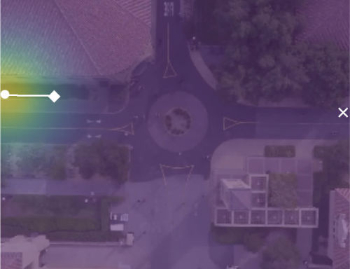}
	\end{minipage}
	\begin{minipage}[c]{0.3\linewidth}
		\includegraphics[width=\linewidth]{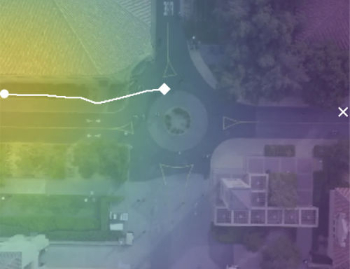}
	\end{minipage}
	\begin{minipage}[c]{0.3\linewidth}
		\includegraphics[width=\linewidth]{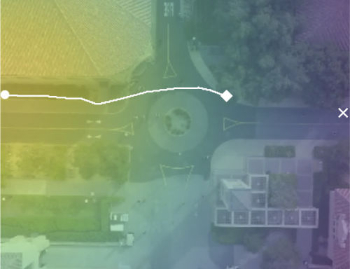}
	\end{minipage}
	
	\vspace{0.2cm}
	\hspace{.5cm}
	\begin{minipage}{0.3\linewidth}
		\centering
		$t = 2.33s$
	\end{minipage}
	\begin{minipage}{0.3\linewidth}
		\centering
		$t = 8.33s$
	\end{minipage}
	\begin{minipage}{0.3\linewidth}
		\centering
		$t = 12.33s$
	\end{minipage}	
	\caption{Scene 2}
	\label{fig:death-1-2}
	\end{subfigure}
	\caption{An illustration of the predictions generated by the various algorithms. In this figure, the dot is the start point of the test trajectory, the diamond is the position at time t, and the X is the end of the trajectory. The likelihood of detection is depicted using the virdis color palette. Notice that the Random Walk is imprecise, while the predictions generated by the algorithm in \cite{Kitani2012} are unable to match the speed of the agent and choose the wrong direction to follow the agent around the circle. The algorithm in \cite{Alahi2016} is confident and close to the trajectory at small times, but their lack of a motion model causes their prediction to compress into a point at intermediate time scales.}
	\vspace{-0.7cm}
\end{figure*}

  \subsection{Learning the Noise Model}
  Finally, we assume that the true position is related to the model by Gaussian noise with a growing variance.
  In particular, we assume $\Pr( x_t \mid \check{x}_t) \sim \mathcal{N}( \check{x}_t , \kappa t)$ for some constant $\kappa \geq 0$.
  The parameter, $\kappa$, must be learned.
  For each curve in $S_k$ we create a synthetic curve using the initial position and speed and integrating the corresponding vector-field, $s\, X_k$.
  So for each curve, $x_i(t)$, of $S_k$, we have a synthesized curve $x_{i,synth}(t)$ as well.
  We then measure the standard deviation of $(x_i(t) - x_{i,synth}(t)) / t$ over $i$ and at few time, $t \in \{ 0, 100, 200 \}$ in order to obtain $\kappa$.
  
 \subsection{Evaluating Performance}
 
This section establishes our methods performance and compares it to the model from \cite{Kitani2012}, \cite{Alahi2016}, and a random walk.
We implement our model as well as our evaluation code in Python 2.6\footnote{https://github.com/ramvasudevan/iros2017\_pedestrian\_forecasting}.
We performed a 2-fold cross validation by using 20\% of the data for testing and the remainder for training within each scene. 
	We learned separate collections of vector fields and model parameters for each fold on all of the four scenes on which we tested.
Our analysis was conducted on the Coupa, Bookstore, Death Circle, and Gates scenes from the dataset from \cite{Robicquet2016}, with a total of 142 trajectories analyzed.


Note that the implementation of the algorithm in \cite{Kitani2012} required the endpoint of each test trajectory. 
Without this information the implementation of the predictor provided by the authors devolved into a random walk.
None of the other tested algorithms required this information.

The output distributions of the four algorithms were compared using their integrals over the cells of a regular grid over our domain. 
These integrals are used to visualize the distributions in Figures \ref{fig:bookstore-1-2} and \ref{fig:death-1-2}.
Because our predictions all share the same scale, we amalgamate all of the prediction and truth values for all of the simulated agents at a given time step and generate ROC curves.
In our analysis, we sought a metric that evaluated the utility of prediction algorithms in the autonomous vehicle context. 
	In this instance it is critical that the generated set of predictions contains the ground-truth observed trajectory while including as few false positive detections as possible.
	ROC curves which plot the true positive rate against the false positive rate evaluate this aforementioned safety criteria for autonomous vehicles exactly. To generate the true positive rate and false positive rate, a probability distribution has to be thresholded to yield a binary map. Each chosen threshold creates a point on the ROC curve.
	The Area Under the Curve (AUC) is a standard measure of the quality of a predictor.  
	The closer that this AUC is to one, the better the prediction algorithm. We treat every value of each bounding box as a threshold.
	Figure \ref{fig:auc_vs_time} shows the analysis of the AUC of each algorithm versus time.
	
In addition, we evaluated the Modified Hausdorff Distance (MHD) \cite{Dubuisson1994} from the ground truth trajectory to a sample from the predictions at each time to provide a geometric measure of how accurate the predictions are. 
	Figure \ref{fig:mhd_vs_time} shows MHD plotted against time.
	Though popular in evaluating the geometric proximity of the ground truth trajectory and a predicted set, it is not the most appropriate way to evaluate the utility of an algorithm in the autonomous vehicle context. 
	Specifically consider the instance of a generated set of predictions which does not contain the ground-truth observed trajectory but is close to the ground-truth trajectory geometrically. 
	If this generated set of predictions was used in the autonomous vehicle context, it would not preclude certain portions of the space where the person was occupying. 
	Notice that the true positive rate would be zero meaning that the generated set of predictions was not useful.
	Whereas the MHD would describe this generated set of predictions as informative.
Our predictor behaves better than any of the other compared methods at moderately large $t$.
This is despite providing the algorithm in \cite{Kitani2012} with the specific final location of each agent.

\begin{table}
	\begin{center}
		\caption{Comparison of runtimes of the various algorithms.}
		\label{tab:time}
		\renewcommand{\arraystretch}{1.5}
		\begin{tabular}{||c | c c  c c ||} 
			\hline
			& Our Algorithm & Random Walk & Kitani et al. & S-LSTM \\ [0.5ex] 
			\hline 
			$\frac{\mathrm{time}}{\mathrm{frame}}$ & 0.00443s & 2.1E-7s & 0.0706s & 0.0134s \\
			\hline
		\end{tabular}
	\end{center}
	\vspace*{-0.6cm}
\end{table}

\begin{figure}
	\centering
	\includegraphics[width=0.9\linewidth]{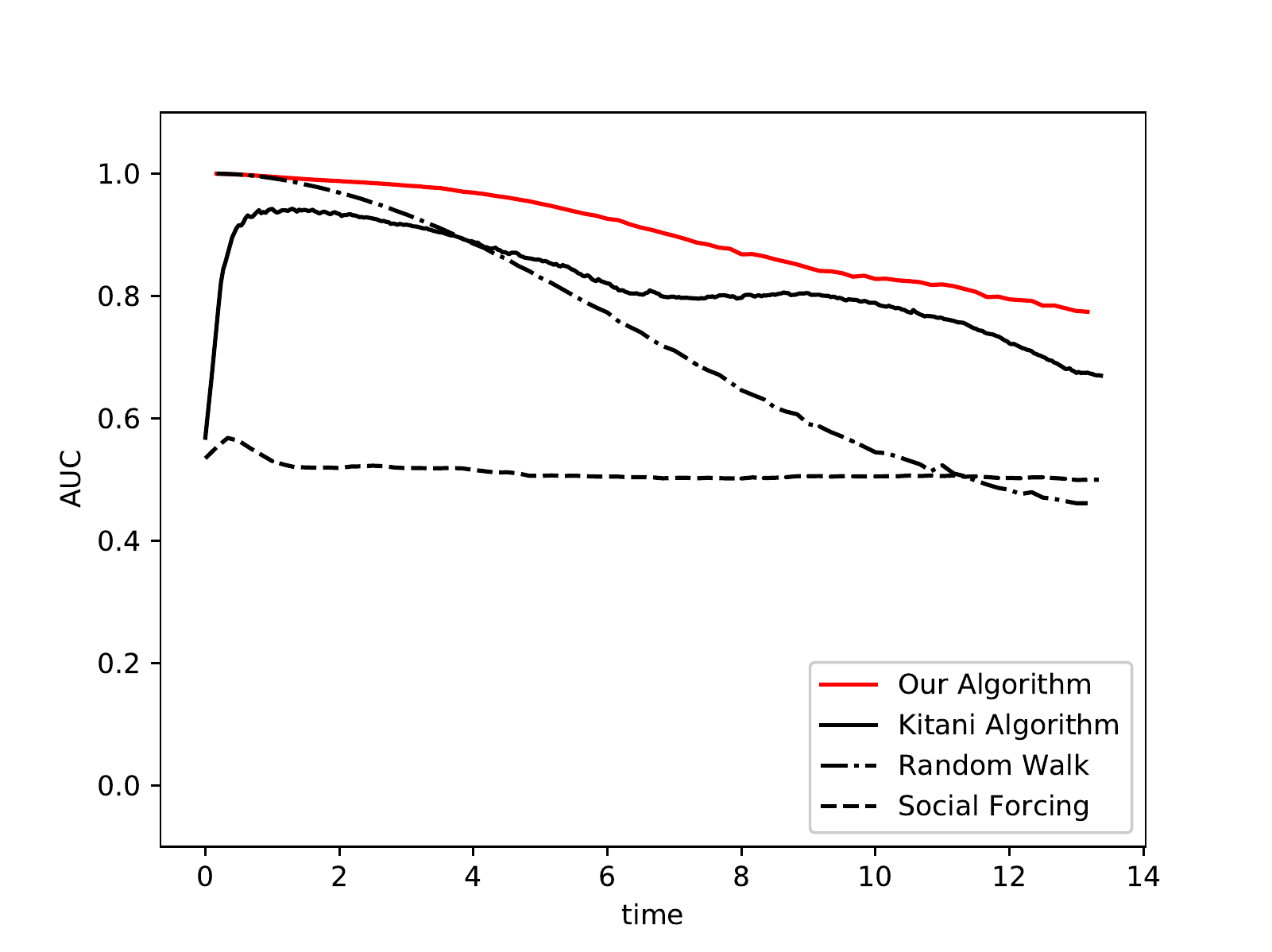}
	\caption{A comparison of the AUC of the various algorithms. Note that the initial dip in the performance of \cite{Kitani2012} is due to their confidence in their initial estimate. We sampled the S-LSTM \cite{Alahi2016} model $100$ times to extract a less concentrated probability distribution from their algorithm.}
	\label{fig:auc_vs_time}
	\vspace*{-0.75cm}
\end{figure}

\begin{figure}
	\centering
	\includegraphics[width=0.9\linewidth]{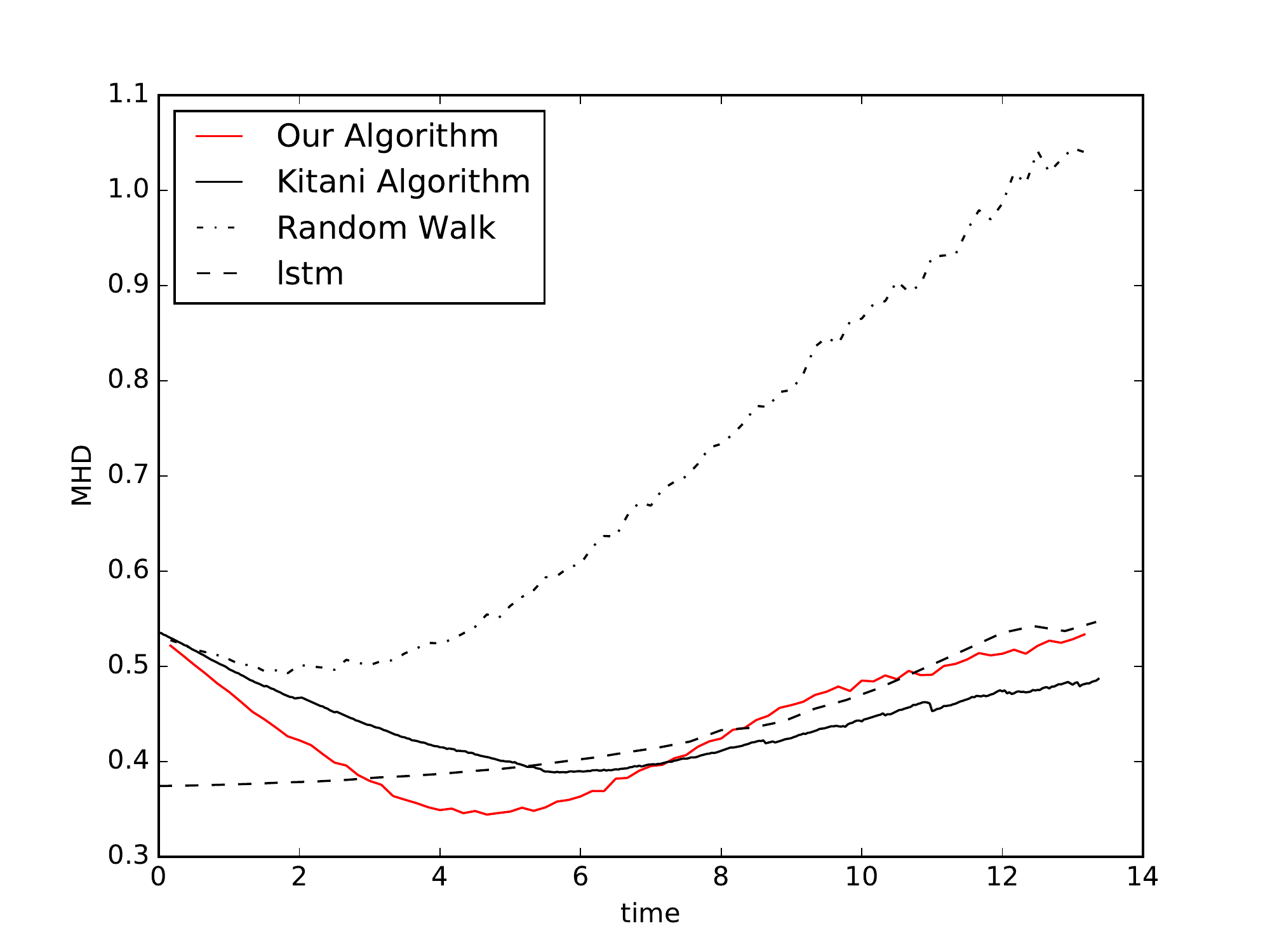}
	\caption{A comparison of the MHD from the ground truth of the pedestrian to a 1000 point samples from each distribution. The method from \cite{Alahi2016} does well at short time scales since it places a highly confident distribution at the given initial position of the pedestrian, but the method developed in this paper outperforms all others at intermediate times. \cite{Kitani2012}, which requires the end point of each trajectory, outperforms all other algorithms at longer time scales since they assume that the end point of the trajectory is known.}
	\label{fig:mhd_vs_time}
	\vspace*{-0.7cm}
\end{figure}

The run time per frame for each algorithm was generated using the mean run time for 400 frames, averaged across all of the trajectory data used in the quality analysis.
This is shown in Table~\ref{tab:time}. 
Our algorithm implementation leveraged its parallelism, which split the computation of frames among 18 cores. 
The algorithm in \cite{Kitani2012} was timed with minimal modification using the source code provided by the authors. 


\section{Conclusion} 
\label{sec:conclusion}
This paper presents a real-time approach to probabilistic pedestrian modeling. 
We demonstrate the ability to accurately predict the final location of a pedestrian with superior accuracy compared to a state-of-the-art with additional contributions in two critical areas: 1) prediction speed is efficient to enable real-time control and 2) our error bound on the predictions accurately reflects potential paths for the human. 
These two features are critical in robotic applications. %
and we see the integration of such techniques with autonomous system control as one pathway to enable safe operation in human-robot spaces. 

Social forces can be incorporated into this model by
	including a social force $F_{i}$ acting on the $i$th agent by letting $\ddot{x}_i = s^2 DX(x) \cdot X(x) + F_i$.  
	Usually $F_{i} = \sum_{j} \nabla U( x_j - x_i)$ where $U$ is an interaction potential \cite{Helbing1995}. 
	The algorithms for generating real-time predictions would then generalize to this definition.
We also plan to explore transfer learning using scene segmentation, as well as the semantic context descriptors and routing scores to show how vector fields can be transferred to novel scenes. It appears that the low-order parameterization and unit-length vector field of our model, make it amenable to the methods developed in \cite{Ballan2016}.

\bibliographystyle{IEEEtran}
\bibliography{references}

\end{document}